\documentclass[11pt,a4paper]{article}
\usepackage{emnlp2017}
\usepackage{hyperref}
\usepackage{url}
\usepackage{times}
\usepackage{latexsym}
\usepackage{array}
\usepackage{multirow}
\usepackage{amssymb}
\usepackage{amsmath}
\usepackage{mathpartir}
\usepackage{mathtools}
\usepackage{amsthm}
\usepackage{xspace}
\usepackage{algorithm}
\usepackage{tikz-dependency}
\usepackage{cancel}
\usepackage{wrapfig}

\usepackage{array,booktabs,ragged2e} 
\newcolumntype{R}[1]{>{\RaggedLeft\arraybackslash}p{#1}}

\usepackage{graphicx}
\usepackage{caption}
\usepackage{subcaption}

\newcommand{\tuple}[1]{\ensuremath{\langle {#1} \rangle}}

\newcommand{\notes}[1]{}

\theoremstyle{definition}

\theoremstyle{plain}
\newtheorem{theorem}{Theorem}

\newcommand{\ith}[1]{\ensuremath{i^{{th}}}}

\newcommand{\ngram}{\ensuremath{\text{$n$-gram}}\xspace}

\newcount\permx
\newcount\permy
\def\permdot#1#2{
\permx=#1 \advance\permx by-1
\permy=#2 \advance\permy by-1
\psframe[fillcolor=black, fillstyle=solid]
(\permx,\permy)(#1, #2)
}

\newcommand{\argmax}{\operatornamewithlimits{\mathbf{argmax}}}

\newcommand{\toptop}{\operatornamewithlimits{\mathbf{top}}}

\newcommand{\startsym}{\mbox{\scriptsize \texttt{<s>}}\xspace}
\newcommand{\stopsym}{\mbox{\scriptsize \texttt{</s>}}\xspace}

\newcommand{\boxnum}[1]{{\setlength{\fboxsep}{1pt}\raisebox{1pt}{\hspace{1pt}\fbox{\tiny #1}\hspace{1pt}}}}
\newcommand{\ind}[1]{\ensuremath{_{\kern-0.5pt\boxnum{#1}}}}

\newcommand{\vecx}{\ensuremath{\mathbf{x}}\xspace}
\newcommand{\vecy}{\ensuremath{\mathbf{y}}\xspace}

\newcommand{\smallnt}[1]{\ensuremath{_{\mbox{\tiny PP}}}\xspace}

\newcommand{\pseudocode}{Algorithm}
\floatname{algorithm}{\pseudocode}

\iffalse

\else

\fi

\newcommand{\defeq}{\ensuremath{\stackrel{\Delta}{=}}\xspace}

\newcommand{\score}{\ensuremath{\mathit{sc}}\xspace}

\newcommand{\pytorch}{PyTorch\xspace}
\newcommand{\rnnsearch}{RNNsearch\xspace}

\newcommand{\bestuptoi}{\ensuremath{\mathit{best}_{\leq i}}\xspace}
\newcommand{\completed}{\ensuremath{\mathit{comp}}\xspace}

\def\namecite{\newcite}

\emnlpfinalcopy

\makeatletter
\def\blfootnote{\gdef\@thefnmark{}\@footnotetext}
\makeatother

\title{When to Finish? Optimal Beam Search for Neural Text Generation (modulo beam size)}

\author{Liang Huang \and Kai Zhao$^\dagger$ \and Mingbo Ma\\
  School of Electrical Engineering and Computer Science\\
  Oregon State University\\
  Corvallis, Oregon, USA\\
  {\tt \{liang.huang.sh, kzhao.hf, cosmmb\}@gmail.com}
}

\begin{document}

\maketitle

\begin{abstract}
In neural text generation such as neural machine translation, summarization, and image captioning,
beam search is widely used to improve the output text quality.
However, 
in the neural generation setting, hypotheses can finish in different steps,
which makes it difficult to decide when to end beam search
to ensure optimality.
We propose a provably optimal beam search algorithm
that will always return the optimal-score complete hypothesis (modulo beam size),
and finish as soon as the optimality is established (finishing no later than the baseline).
To counter neural generation's tendency for shorter hypotheses,
we also introduce a bounded length reward mechanism which allows 
a modified version of our beam search algorithm to remain optimal.
Experiments on neural machine translation 
demonstrate that our principled beam search algorithm
leads to improvement in BLEU score over previously proposed alternatives.
\end{abstract}

\blfootnote{$^\dagger$ Current address: Google Inc., 
New York, NY, USA.}

\section{Introduction}
\label{sec:intro}

In recent years, neural text generation using recurrent networks
have witnessed rapid progress, quickly becoming the state-of-the-art paradigms 
in machine translation \cite{kalchbrenner+blunsom:2013,sutskever+:2014,bahdanau+:2014},
summarization \cite{rush+:2015,ranzato+:2016}, and
image captioning \cite{vinyals+:2015,xu+:2015}.
In the decoder of neural generation, beam search is widely employed
to boost the output text quality, often leading to
substantial improvement over greedy search (equivalent to beam size 1) in metrics such as BLEU or ROUGE;
for example, \namecite{ranzato+:2016}
reported +2.2 BLEU (on single reference) in translation
and +3.5 ROUGE-2 in summarization, both using a beam of 10.
Our own experiments on machine translation (see Sec.~\ref{sec:exps})
show +4.2 BLEU (on four references) using a beam of 5.

However, unlike traditional beam search in phrase-based MT or shift-reduce parsing
where all hypotheses finish in the same number of steps,
here in neural generation, hypotheses can finish in vastly different numbers of steps.
Once you find a completed hypothesis (by generating the \stopsym symbol),
there are still other active hypotheses in the beam that can continue to grow,
which might lead to better scores.
Therefore when can you end the beam search? 
How (and when) can you guarantee that the returned hypothesis has the optimal score
modulo beam size?

There have not been satisfying answers to these questions, 
and existing beam search strategies are heuristic methods that do not guarantee optimality. 
For example, the widely influential \rnnsearch \cite{bahdanau+:2014} employs a ``shrinking beam'' method:
once a completed hypothesis is found, beam size shrinks by 1,
and beam search would finish if beam size shrinks to 0 or if the number of steps hits a hard limit.
The best scoring completed hypothesis among all completed ones encountered so far is returned.
On the other hand, OpenNMT \cite{opennmt}, whose PyTorch version will be the baseline in our experiments,
uses a very different strategy: beam search terminates whenever the highest-ranking hypothesis in the current step is completed
(which is also the one returned), without considering any other completed hypotheses.
Neither of these two methods  guarantee optimality of the returned hypothesis.

We therefore propose a novel and simple beam search variant
that will always return the optimal-score complete hypothesis (modulo beam size),
and finish as soon as the optimality is established.
However, another well-known problem remains, that
the generated sentences are often too short, compared to previous paradigms such as SMT \cite{shen+:2016}.
To alleviate this problem, previous efforts introduce length normalization (as a switch in RNNsearch)
or length reward \cite{he+:2016} borrowed from SMT \cite{koehn+:2007}.
Unfortunately these changes will invalidate the optimal property of our proposed algorithm.
So we introduce a {\em bounded} length reward mechanism which allows 
a modified version of our beam search algorithm to remain optimal.
Experiments on neural machine translation 
demonstrate that our principled beam search algorithm
leads to improvement in BLEU score over previously proposed alternatives.


\section{Neural Generation and Beam Search}
\label{sec:prelim}

Here we briefly review neural text generation 
and then review existing  beam  search algorithms.

Assume the input sentence, document, or image is
embedded into a vector \vecx, from which we generate the output sentence \vecy which is a {\em completed} hypothesis:\footnote{For simplicity reasons we do not discuss bidirectional LSTMs and attentional mechanisms here 
but our algorithms still work with those encoders (we have tested them).}
\begin{align*}
  \vecy^* &= \argmax_{\vecy: \completed(\vecy)} p( \vecy \mid \vecx) \\
          &= \argmax_{\vecy: \completed(\vecy)} \prod_{i\leq |\vecy|} p(y_i \mid \vecx, \vecy_{<i})
\end{align*}
where $\vecy_{<i}$ is a popular shorthand notation for the prefix $y_0 y_1 ... y_{i-1}$.
We say that a hypothesis \vecy is {\bf completed}, notated $\completed(\vecy)$, if its last word is \stopsym, i.e.,
\[
        \completed(\vecy) \defeq (\vecy_{|\vecy|} = \stopsym)
\]
in which case it will not be further expanded.

A crucial difference in RNN-based neural generation compared to previous paradigms such as phrase-based MT
is that we no longer decompose $p(y_i \mid \vecx, \vecy_{<i})$ 
into the translation model, $p(y_i \mid \vecx)$, and the language model, $p(y_i \mid \vecy_{<i})$, 
and more importantly, we no longer approximate the latter by \ngram models.
This ability to model arbitrarily-lengthed history using RNNs is an important reason for NMT's substantially improved fluency
compared to SMT.

To (approximately) search for the best output $\vecy^*$,
we use beam search, where the beam $B_i$ at step $i$ is an {\it ordered} list of size (at most) $b$, 
and expands to the next beam $B_{i+1}$ of the same size:
\begin{align*}
B_0\! &= \![\tuple{\startsym, \ p(\startsym \mid \vecx)}] \\
B_i\! &= \!\toptop^b  
     \{\tuple{\vecy'\!\!\circ y_i, \ s\!\cdot\! p(y_i | \vecx, \vecy)} \mid \tuple{\vecy'\!, s} \in B_{i-1} \}
\end{align*}
where the notation $\toptop^b S$ selects the top $b$ scoring items from the set $S$,
and each item is a pair $\tuple{\vecy, s}$ where \vecy is the current prefix and $s$ is its accumulated score
(i.e., product of probabilities).

\section{Optimal Beam Search {\normalfont \small (modulo beam size)}}
\label{sec:beam}
We propose a very simple method to optimally finish beam search,
which guarantees the returned hypothesis is the highest-scoring 
completed hypothesis modulo beam size;
in other words, we will finish as soon as an ``optimality certificate''
can be established that future hypotheses will never score better
than the current best one.

Let \bestuptoi be the {\em best completed hypothesis so far} up to step $i$, i.e.,
\begin{equation}
\bestuptoi \defeq \max \{ \vecy \in \cup_{j\leq i} B_j \mid \completed(\vecy) \}
\end{equation}
We update it every time we find a completed hypothesis (if there is none yet, then it remains undefined).
Now at any step $i$, if \bestuptoi is defined, and the highest scoring item $B_{i,1}$ in the current beam $B_i$
scores worse than or equal to \bestuptoi, i.e., when
\begin{equation}
B_{i,1} \leq \bestuptoi
\end{equation}
we claim the optimality certificate is established, and terminate beam search, returning \bestuptoi
(here smaller means worse, since we aim for the highest-probability completed hypothesis).

\begin{theorem}[optimality]
When our beam search algorithm terminates,
the current best completed hypothesis (i.e., \bestuptoi) is the highest-probability completed hypothesis (modulo beam size).
\end{theorem}
\begin{proof}
If $B_{i,1} \leq \bestuptoi$ then $B_{i,j} \leq B_{i,1} \leq \bestuptoi$ for all items $B_{i,j}$ in beam $B_i$.
Future descendants grown from these items will only be no better, since probability $\leq 1$,
so all items in current and future steps are no better than \bestuptoi.
\end{proof}

\begin{theorem}[early stopping]
Our beam search algorithm terminates no later than OpenNMT's termination criteria
(when $B_{i,1}$ is completed).
\end{theorem}
\begin{proof}
When $B_{i,1}$ is itself completed, $\bestuptoi=\max\{B_{i,1}, \cdots\} \geq B_{i,1}$,
so our stopping criteria is also met.
\end{proof}

This above Theorem shows that our search is stopping earlier once the optimality certificate is established, exploring fewer items than OpenNMT's default search.
Also note that the latter, even though exploring more items than ours,
still can return suboptimal solutions; e.g.,
when $B_{i,1}$ is worst than \bestuptoi (they never stored \bestuptoi).
In practice, we noticed our search finishes about 3--5 steps earlier than OpenNMT at a beam of 10,
and this advantage widens as beam size increases, although the overall speedup is not too noticeable,
given the target language sentence length is much longer.
Also, our model scores (i.e., log-probabilities) are indeed better (see Fig.~\ref{fig:score-beam}),
where the advantage is also more pronounced with larger beams
(note that OpenNMT baseline is almost flat after $b=10$, while our optimal beam search still steadily improves).
Combining these two Theorems, it is interesting to note that our method is not just optimal but also faster.

\begin{figure}
\centering\hspace{-0.3cm}\includegraphics[width=0.51\textwidth]{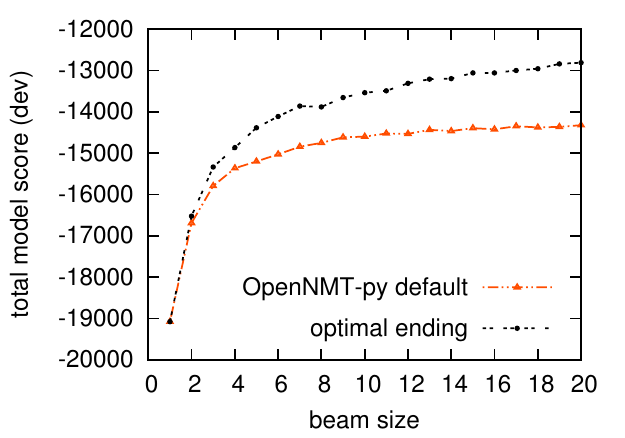} 
\caption{
Comparison between optimal beam search and OpenNMT-py's default search,
in terms of search quality (model score, $\uparrow$ is better).
\label{fig:score-beam}}
\end{figure}

\section{Optimal Beam Search for Bounded Length Reward}
\label{sec:length}
However, optimal-score hypothesis, though satisfying in theory,
is not ideal in practice, since neural models
are notoriously bad in producing very short sentences, 
as opposed to older paradigms such as SMT \cite{shen+:2016}.
To alleviate this problem, two methods have been proposed:
(a) length normalization, used in \rnnsearch as an option,
where the revised score of a hypothesis is divided by its length, thus favoring longer sentences;
and (b) explicit length reward \cite{he+:2016} borrowed from SMT,
rewarding each generated word by a constant tuned on the dev set.

Unfortunately, each of these methods breaks the optimality proof of our beam search algorithm in Section~\ref{sec:beam},
since a future hypothesis, being longer, might end up with a higher (revised) score.
We therefore devise a novel mechanism called ``bounded length reward'',
that is, we reward each word until the length of the hypothesis is longer than the ``estimated optimal length''.
In machine translation and summarization, this optimal length $l$ can be $\mathit{ratio} \cdot |\vecx|$ where 
$|\vecx|$ is the source sentence length, and $\mathit{ratio}$ is the 
average ratio of reference translation length over source sentence length on the dev set
(in our Chinese-to-English NMT experiments, it is 1.27 as the English side is a bit longer).
Note that we use the same $\mathit{ratio}$ estimated from dev on test,
assuming that the optimal length ratio for test (which we do not know) 
should be similar to those of dev ones.
We denote $\tilde{\score}({\vecy})$ to be the revised score of hypothesis $\vecy$ with the bounded length reward, i.e.,
\[
\tilde{\score}({\vecy}) \defeq \score(\vecy) + r \cdot \min \{l, |\vecy|\}.
\]
We also define $\tilde{\bestuptoi}$ to be the revised version of \bestuptoi that optimizes the revised instead of the original score, i.e.,
\[
\displaystyle\tilde{\bestuptoi} \defeq \argmax_{\vecy \in \cup_{j\leq i} B_j, \completed(\vecy)} \tilde{\score}(\vecy)
\]

Now with bounded length reward, we can modify our beam search algorithm a little bit 
and still guarantee optimality.
First we include in the revised cost a reward $r$ for each generated word,
as long as the length is less than $l$, the estimated optimal length.
If at step $i$, the highest scoring item $B_{i,1}$'s revised score (i.e., including bounded length reward) 
plus the heuristic ``future'' extra length reward of a descendant,
\(r\cdot \max \{l-i, 0\}\),
is worse than (or equal to) the similarly revised version of \bestuptoi, i.e.,
\begin{equation}
\label{eq:comp}
\tilde{\score}({B}_{i,1}) + r\cdot \max\{l-i, 0\} \leq  \tilde{\score}(\tilde{\bestuptoi})
\end{equation}
at which time we claim the revised optimality certificate is established,
and terminate the beam search and return $\tilde{\bestuptoi}$.

Actually with some trivial math we can simplify the stopping criteria to
\begin{equation}
\label{eq:simp}
        \score(B_{i,1}) + r\cdot l \leq \tilde{\score}({\tilde{\bestuptoi}}).
\end{equation}
This much simplified but still equivalent criteria can speed up decoding in practice,
since this means we actually do not need to compute the revised score for every hypothesis in the beam;
we only need to add the bounded length reward when one is finished (i.e., when updating 
${\tilde{\bestuptoi}}$), and the simplified criteria only compares it with the original score of a hypothesis
plus a constant reward $r\cdot l$.

\begin{theorem}[modified optimality]
Our modified beam search returns the highest-scoring completed hypothesis 
where the score of an item is its log-probability plus a bounded length reward.
\end{theorem}
\begin{proof}
by admissibility of the heuristic.
\end{proof}

\begin{theorem}[correctness of the simplified criteria]
Eq.~\ref{eq:simp} is equivalent to Eq.~\ref{eq:comp}.
\end{theorem}
\begin{proof}
trivial.
\end{proof}


\section{Experiments: Neural Translation}
\label{sec:exps}

\subsection{\!Data Preparation, Training, and Baselines}

\begin{table}
\centering
\scalebox{0.9}{
\begin{tabular}{l|rrrrrr}
        & sents & tokens & vocab. & w/ BPE \\
\hline
Chinese & 1M    &  28M   & 112k   & 18k\\
English & 1M    &  23M   &  93k   & 10k \\
\end{tabular}
}
\caption{Machine translation training set.\label{tab:train}}
\end{table}

\begin{figure*}
\begin{tabular}{cc}
\hspace{-0.1cm}\includegraphics[width=0.5\textwidth]{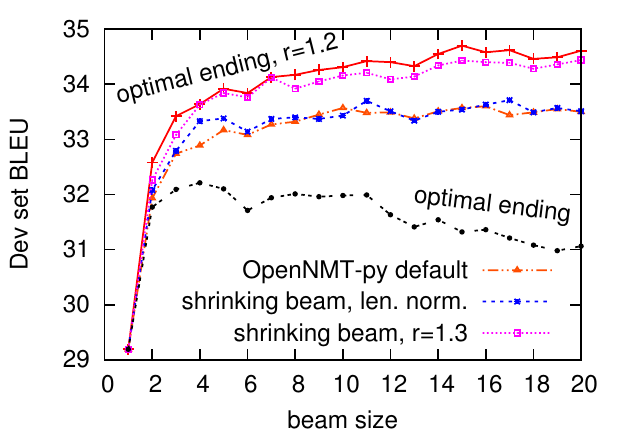} 
&
\hspace{-0.1cm}\includegraphics[width=0.5\textwidth]{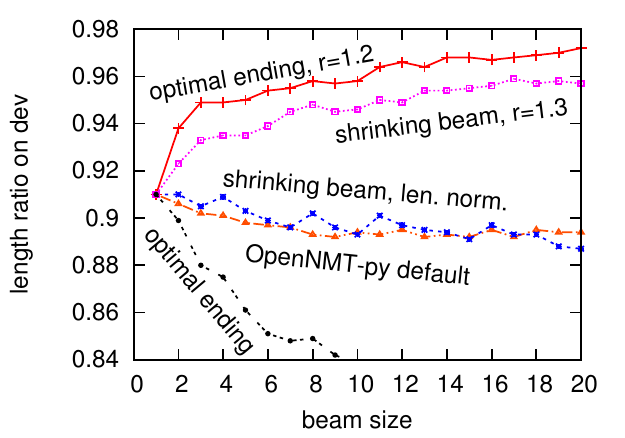} 
\\
(a) BLEU vs.~beam size 
&
(b) length ratio vs.~beam size
\end{tabular}
\caption{BLEU score and length ratio against beam size (on dev) of various beam search algorithms for neural machine translation.\label{fig:bleu-beam}}
\end{figure*}

We conduct experiments on Chinese-to-English neural machine translation,
using OpenNMT-py,\footnote{\small\url{https://github.com/opennmt/opennmt-py}}
the \pytorch port of the Lua-based OpenNMT~\cite{opennmt}.
We choose this library because \pytorch's combination of Python with Torch's dynamic computation graphs made it much easier to implement various search algorithms 
on it than on Theano-based implementations derived from \rnnsearch \cite{bahdanau+:2014}
(such as the widely used GroundHog\footnote{\small\url{https://github.com/lisa-groundhog/}} and Laulysta\footnote{\small\url{https://github.com/laulysta/nmt/}} codebases)
as well as the original LuaTorch version of OpenNMT.
We use 1M Chinese/English sentence pairs for training (see Table~\ref{tab:train} for statistics); we also trained on 2M sentence pairs and only saw a minor improvement so below we report results from 1M training.
To alleviate the vocabulary size issue we employ byte-pair encoding (BPE) \cite{sennrich+:2015}
which reduces the source and target language  vocabulary sizes to 18k and 10k, respectively;
we found BPE to significantly improve BLEU scores (by at least +2 BLEU) and reduce training time.
Following other papers on Chinese-English translation
such as \namecite{shen+:2016}, we use NIST 06 newswire portion (616 sentences) for development
and NIST 08 newswire portion (691 sentences) for testing;
we will report case-insensitive 4-reference BLEU-4 scores (using original segmentation).

Following OpenNMT-py's default settings, we train our NMT model for 20 epochs to minimize perplexity on the training set
(excluding 15\% sentences longer than 50 source tokens), with a batch size of 64, word embedding size of 500, and dropout rate of 0.3.
The total number of parameters is 29M. Training takes about an hour per epoch on Geforce 980 Ti GPU, 
and the model at epoch 15 reaches the lowest perplexity on the dev set (9.10) 
which is chosen as the model for testing.

On dev set, the default decoder of OpenNMT-py reaches 29.2 BLEU with beam size~1 (greedy)
and 33.2 BLEU with the default beam size of~5.
To put this in perspective, the most commonly used SMT toolkit Moses \cite{koehn+:2007} reaches 30.1 BLEU (with beam size 70) using the same 1M sentence training set
(trigram language model trained on the target side).
With 2.56M training sentence pairs, \namecite{shen+:2016} reported 32.7 BLEU on the same dev set using Moses
and 30.7 BLEU using the baseline \rnnsearch (GroundHog) with beam size 10 (without BPE, without length normalization or length reward).  
So our OpenNMT-py baseline is extremely competitve.

\subsection{Beam Search \& Bounded Length Reward}

\begin{table}
\scalebox{0.9}{
\begin{tabular}{p{1.35cm}R{0.55cm}R{0.55cm}R{0.55cm}R{0.55cm}R{0.55cm}R{0.55cm}R{0.55cm}}
reward $r$ & 0 & 1 & 1.1 & 1.2 & 1.3 & 1.4 & 1.5 \\
\hline
BLEU & 32.2 & 34.6 & 34.6 & {\bf 34.7} & 34.6 & 34.6 & 34.6    \\
len.~ratio & 0.88 & .95 & .96 & .97 & .98 & .98 & .99\\
best $b$ & 4 & 17 & 17 & 15 & 20 & 20 & 17 \\
\end{tabular}
}
\caption{Tuning length reward $r$ (with beam size $b$=1..20) for optimal bounded-reward beam search.
\label{tab:tune}}
\end{table}

We compare the following beam search variants:
\begin{enumerate}
\item OpenNMT-py's default beam search, finishing only when the top hypothesis in a step is completed (see Section~\ref{sec:prelim});
\item The ``shrinking beam'' method in \rnnsearch
  with two variants
  to encourage longer translations:
  \begin{enumerate}
    \item length normalization; 
      Google NMT \cite{googlenmt:2016} also adopted a similar mechanism.
    \item unbounded length reward (tuned on dev set) in Baidu NMT \cite{he+:2016}.
  \end{enumerate}
\item Our optimal-ending beam search (Section~\ref{sec:beam});
\item Our modified optimal-ending beam search for bounded length reward (Section~\ref{sec:length}).  
\end{enumerate}

Notice that length reward has no effect on both methods 1 and 2(a) above.
To tune the optimal length reward $r$ we run our modified optimal-ending beam search algorithm
with all combinations of $r=0, 0.5, 1, 1.1, 1.2, 1.3, 1.4$ with beam sizes $b=1\ldots 20$
on the dev set,
since different beam sizes might prefer different length rewards.
We found $r=1.2$ to be the best among all length rewards (see Table~\ref{tab:tune})
which is used in Figure~\ref{fig:bleu-beam} 
and $b=15$ is the best for $r=1.2$.

We can observe from Figure~\ref{fig:bleu-beam} that
(a) our optimal beam search with bounded length reward
performs the best, and at $b$=15 it is +5 BLEU better than $b$=1;
(b) pure optimal beam search degrades after $b$=4 due to
extremely short translations;
(c) both the shrinking beam method with length normalization 
and OpenNMT-py's default search alleviate the shortening problem,
but still produce very short translations (length ratio $\sim$0.9).
(d) the shrinking beam method with length reward works well,
but still 0.3 BLEU below our best method.
These are confirmed by the test set (Tab.~\ref{tab:final}).

\begin{table}
\hspace{-.2cm}
\begin{tabular}{r|rrrr}
decoder    & $b$ & dev & test\\
\hline
Moses & 70 & 30.14 & 29.41 \\
\hline
OpenNMT-py default & 16 & 33.60 & 29.75 \\
shrinking, len.~norm. & 17 & 33.71 & 30.11 \\
shrinking, reward $r$=1.3 & 15 & 34.42 & 30.37\\
\hline
optimal beam search, $r$=1.2 & 15 & {\bf 34.70} & {\bf 30.61} \\
\end{tabular}
\caption{Final BLEU scores on the test set (nist 08)
  using best settings from the dev set (nist 06).
\label{tab:final}
}
\end{table}


\section{Conclusions}
We have presented a beam search algorithm for neural sentence generation
that always returns optimal-score completed hypotheses.
To counter neural generation's natural tendancy for shorter hypotheses,
we introduced a {\em bounded length reward} mechanism which allows 
a modified version of our beam search algorithm to remain optimal.
Experiments on top of strong baselines have confirmed that our principled search algorithms (together with our bounded length reward mechanism) 
outperform existing beam search methods in terms of BLEU scores.
We will release our implementations (which will hopefully be merged into OpenNMT-py)
when this paper is published.
\footnote{While implementing our search algorithms
we also found and fixed an obscure but serious bug in OpenNMT-py's baseline beam search code (not related to discussions in this paper),
which boosts BLEU scores by about +0.7 in all cases.
We will release this fix as well.}

\section*{Acknowledgments}
{
We thank the anonymous reviewers from both EMNLP and WMT for helpful comments.
This work is supported in part by
NSF IIS-1656051,
DARPA N66001-17-2-4030 (XAI),
a Google Faculty Research Award,
and HP.
}

\bibliography{thesis}
\bibliographystyle{emnlp_natbib}

\end{document}